\documentclass[journal,11pt]{IEEEtran}
\ifCLASSINFOpdf
\else
\fi
\usepackage{comment} 

\usepackage{amsmath,graphicx}

\usepackage{pstricks,graphicx}
\usepackage{setspace}
\usepackage{psfrag}
\usepackage{amssymb}
\usepackage{amsmath}
\usepackage{amsthm}

\newtheorem{remark}{Remark}

\newtheorem{lemma}{Lemma}

\newtheorem{theorem}{Theorem}
\newtheorem{proposition}{Proposition}

\newtheorem{definition}{Definition}

\usepackage{microtype}
\usepackage{graphicx}
\usepackage{subfigure}
\usepackage{booktabs} % for professional tables

 \usepackage{algorithm}
 \usepackage{algorithmic}

\usepackage{amsmath,graphicx}

\usepackage{pstricks,graphicx}
\usepackage{setspace}
\usepackage{psfrag}
\usepackage{amssymb}
\usepackage{amsmath}
\usepackage{amsthm}

\usepackage{hyperref}

\begin{document}

%\onecolumn

%\doublespacing 

\title{Rapidly Adapting Moment Estimation}

\graphicspath{{RAME_figures/}}

\author{Guoqiang~Zhang,  Kenta Niwa and W. B. Kleijn
\thanks{G.~Zhang is with the School of Electrical and Data Engineering, University of Technology, Sydney, Australia. Email: {guoqiang.zhang@uts.edu.au}}
\thanks{K.~Niwa is with Nippon Telegraph and Telephone (NTT) Corporation, Japan.
Email: {niwa.kenta@lab.ntt.co.jp}}
\thanks{W. B. Kleijn is with the School of Engineering and Computer Science, Victoria University of Wellington, New Zealand. Email: {bastiaan.kleijn@ecs.vuw.ac.nz}}
%\thanks{Part of the work has been published on ICASSP, 2015, with the paper titled \emph{Bi-Alternating Direction Method of Multipliers over Graphs}. After careful consideration, we decide to change the name of our algorithm from \emph{bi-alternating direction method of multipliers (BiADMM)} in \cite{xiaoqiang14BiADMM} and \cite{xiaoqiang15BiADMM} to \emph{primal-dual method of multipliers (PDMM)}.
}

\maketitle
\begin{abstract}

Adaptive gradient methods such as Adam have been shown to be very effective for training deep neural networks (DNNs) by tracking the second moment of gradients to compute the individual learning rates. Differently from existing methods, we make use of the most recent first moment of gradients to compute the individual learning rates per iteration. The motivation behind it is that the dynamic variation of the first moment of gradients may  provide useful information to obtain the learning rates.  We refer to the new method as the \emph{rapidly adapting moment estimation (RAME)}. % is By doing so, the new method is able to react to the dynamic variation of the first moment of gradients responsively, which is referred to as the \emph{responsively adaptive moment estimation (Rame)}. %The intensities of the individual learning rates of Rame are controlled by a scalar parameter, which includes the heavy-ball method as a special case by setting the scalar parameter to zero.  
%One advantage of RAME is that it saves a memory space of the model size in comparison to Adam by avoiding second moment storage, which becomes significant for large-scale DNNs.  
The theoretical convergence of deterministic RAME is studied by using an analysis similar to the one used in \cite{Adam18Converge} for Adam.  Experimental results for training a number of DNNs show promising performance of RAME w.r.t. the convergence speed and generalization performance compared to the stochastic heavy-ball (SHB) method,  Adam, and RMSprop. % The empirical study confirms that the first moment of gradients indeed helps with achieving either better or equivalent  performance than SHB. 

%We then consider applying the new algorithm for distributed averaging. For the case of no transmission failure, the new algorithm remarkably outperforms the state-of-the-art methods. For the case of transmission losses, the new algorithm is robust to transmission-failure.
\end{abstract}



% \fi
\IEEEpeerreviewmaketitle

\vspace{0mm}
\section{Introduction}
\vspace{0mm}
Stochastic gradient descent (SGD) and its variants have been widely applied in deep learning due to their simplicity and effectiveness \cite{Lecun15nature}.  Vanilla SGD (i.e., without making use of the gradient trajectory) often works reasonably well given enough time if the learning rate is set properly in a dynamical manner over the training iterations. Generally speaking, the historical gradients of SGD carry information about the local problem structure, such as curvature and individual noise levels of current gradient coordinates. Therefore, it is natural to exploit historical gradients to assist the current parameter update for fast convergence. 

In the literature, significant progress has been achieved on making use of historical gradients to accelerate vanilla SGD. Suppose the objective function $f(\boldsymbol{x})$ is differentiable. In 1964, Polyak proposed the so-called heavy-ball (HB) method for minimizing the objective function \cite{Polyak64CM}, which is given by%\footnote{ The expression (\ref{equ:SGDM1})-(\ref{equ:SGDM2}) is extracted from the pytorch platform. There are other forms of SGDM in the literature such as the one in Keras platform. }
\begin{align}
\boldsymbol{m}_{t} & = \beta_t \boldsymbol{m}_{t-1} + \alpha_t  \nabla f(\boldsymbol{x}_{t-1}) \label{equ:SGDM1}\\
\boldsymbol{x}_{t} &= \boldsymbol{x}_{t-1} - \eta_t \boldsymbol{m}_{t},  \label{equ:SGDM2}
\end{align}
where $\nabla f(\boldsymbol{x}_{t-1})$ is the gradient at $\boldsymbol{x}_{t-1}$, and $\alpha_t$ (or $\eta_t$)\footnote{The Keras platform treats $\alpha_t$ as the learning rate and set $\eta_t=1$ while Pytorch takes $\eta_t$ as the learning rate and set $\alpha_t=1$. }  is the common learning rate for all the coordinates of $\boldsymbol{x}_t$.  Later, in 1983, Nesterov proposed a method to further accelerate HB by making  use of the first moments in a smart way \cite{Nesterov83NAG, Sutskever13NAG, Yang2016HB}, which is known as Nesterov's accelerated gradient (NAG). Considering HB, we note from (\ref{equ:SGDM2}) that $\boldsymbol{x}_{t}$ is updated as a  linear function of the first moment $\boldsymbol{m}_{t}$. To our best knowledge,  \emph{there is no prior work on designing a nonlinear function of the first moment $\boldsymbol{m}_t$ for a more effective parameter update.} In this work, we will attempt to do so, where the nonlinearity of  $\boldsymbol{m}_t$ will be interpreted as a form of individual learning rates as opposed to the common learning rate  $\alpha_t$ (or $\eta_t$).  

In the last decade, research on computing proper individual learning rates for $\boldsymbol{x}_t$ in SGD has made considerable progress. Duchi et. al \cite{Duchi11AdaGrad}, in 2011, were first to propose the tracking of the second moment of the gradients. The resulting method, Adagrad, computes the gradient based on the tracked information. It is found that AdaGrad converges fast when the gradients are sparse. Following the work of \cite{Duchi11AdaGrad}, various adaptive gradient methods have been proposed for computing more effective individual learning rates. The methods include, for example, RMSprop \cite{Tieleman12RMSProp}, Adam \cite{Kingma17},  NAdam \cite{Dozat16NAdam}, AMSGrad \cite{Reddi18Amsgrad}, and PAdam \cite{Chen18Padam}.  
We note that all the above methods need to track a certain form of the second moment of gradients.
%The above algorithms are referred to as \emph{adaptive gradient methods} due to the adaptive computation of the individual learning rates.  

While deep learning has seen rapid advances in algorithmic development, theoretical convergence analysis has also made remarkable progress recently.  The work of \cite{Reddi18Amsgrad} showed that Adam does not converge for a special class of convex optimization problems.  The authors of \cite{Adam18Converge} studied the convergence of Adam and RMSprop for smooth nonconvex optimization.  \cite{Zhou18con_PAdam} and  \cite{Chen18con_Adam} also considered smooth nonconvex optimization. In particular,  \cite{Zhou18con_PAdam}  analyzed the convergence of PAdam while \cite{Chen18con_Adam} considered AMSGrad and a variant of AdaGrad. From a high level point of view, analysis of nonconvex optimization is highly valuable in practice as training a deep neural network (DNN) is well known to be a nonconvex optimization problem.  %An improved convergence analysis on existing algorithms would provide insights on designing more advanced adaptive gradient methods.   

In this work, we propose a new adaptive gradient method based on a novel design principle.  In the new method, the individual learning rates are computed by using only the most recent first moment. By doing so, the method is able to react to the dynamic variation of the first moment rapidly, which is why it is referred to as \emph{rapidly adapting moment estimation (RAME)}. Our motivation for the new algorithm development is based on the hypothesis that the first moment may already carry useful information %of the dynamic variation of gradients 
to allow for the learning-rate computation.  If the first moment is available, it may not be needed to compute the second moment, thus saving a memory space of the DNN model size. 

As is summarized in Alg.1, RAME is designed by using a nonlinear function $\boldsymbol{m}_t/(|\boldsymbol{m}_t|^q+\xi)$ of the first moment $\boldsymbol{m}_t$ for the parameter update. The nonlinear function makes the heavy-ball (HB) method less heavy.  
%The form of the individual learning rates $\{1/|(\boldsymbol{m}_t|^q+\xi) \}$ in Alg.1 is partially inspired by the observations made in \cite{Jastrzebski19SGD} for studying vanilla SGD based on the Hessian matrix information (see Subsection~\ref{subsec:alg}). %It is found emprically in \cite{Jastrz?bski19SGD} that if the  . 
 With the expression $1/(|\boldsymbol{m}_t|^q+\xi) $,  the moment coordinates of $\boldsymbol{m}_t$ with large magnitudes receive small learning rates while those with small magnitudes are equipped with relatively large learning rates.   %The parameter $q$ controls the intensities of the individual learning rates of RAME. %The introduced individual learning rates $\{1/|(\boldsymbol{m}_t|^q+\xi) \}$ of RAME are controlled by a scalar parameter $q$. 
%The new method reduces to HB when $q$ approaches to zero.  
To better understand the impact of the expression $1/|(\boldsymbol{m}_t|^q+\xi) $, we reformulate and interpret its update expressions from a dynamic system perspective. Its convergence is studied by using an analysis that is similar to that in  \cite{Adam18Converge}  for deterministic Adam.

We evaluate RAME together with stochastic HB (SHB), Adam, and RMSprop for both classification and regression problems in deep learning.  Specifically,  four  classification tasks are investigated, which are training VGG16 \cite{Simonyan16DCNN} for CIFAR10 and CIFAR100, training ResNet20 \cite{He15ResNet} for CIFAR10, and training a  multiple layer perceptron (MLP) network for CIFAR10.  As for regression, we conduct people semantic segmentation using ResNet152 as the backend \cite{He15ResNet, Lin2016Pyramid}. The Microsoft COCO database is employed to train the neural network. The convergence results obtained from the above tasks show that RAME produces either better or equivalent validation performance compared to the other three methods.   
% for segmenting or differentiating all people from background in an image on a pixel level. 
%Ideally we want to say it is cheap and provides good test performance.

%Experimental results for training VGG16 over CIFAR10 and CIFAR100 show that RAME produces lower training loss and better generalization performance than the other three methods. ResNet20 is also tested for CIFAR10, and it is found that the validation accuracy of RAME, Adam and SHB are roughly the same while RAME again produces lower training loss. For completeness, we also conduct algorithmic comparison for training a multiple layer perceptron (MLP) network over CIFAR10. The results show that RAME converges faster than the other methods and produces promising validation accuracy. Regarding regression, we consider people semantic segmentation using ResNet152 as the backend \cite{Lin2016Pyramid}. The Microsoft COCO database is employed to fine-tune the heavy neural network for segmenting or differentiating all people from background in an image on a pixel level. Experimental results demonstrate that RAME outperforms Adam and SHB w.r.t. both the convergence speed and generalization performance.     

The remainder of the paper is organized as follows:  Section~\ref{sec:problem} introduces notations and defines the optimization problem.  Section~\ref{sec:RAME} is devoted to the new method RAME. In Section~\ref{sec:reformulation}, we provide a new interpretation of RAME from a dynamic system viewpoint. Section~\ref{sec:analysis} presents the algorithmic convergence analysis. After that,  experimental results are then described in Section \ref{sec:experiment}, followed by conclusions in  Section~\ref{sec:conclusion}. 

\vspace{0mm}
\section{Notations and Problem Definition}
\vspace{-0mm}
\label{sec:problem}

\begin{algorithm}[t]
   \caption{RAME for a deterministic function $f(\boldsymbol{x})$}
   \label{alg:v2}
\begin{algorithmic}[1]
   \STATE {\bfseries Input:} $\beta_t$, $\eta_t$,  $\alpha_t$, $1> q \geq 0$, $\xi>0$
   \STATE {\bfseries  Init.:} $\boldsymbol{x}_0\in \mathbb{R}^d$,  $\boldsymbol{m}_0 = 0$
   \FOR{$t=1, 2, \ldots, T$}
   \STATE $\boldsymbol{g}_t \leftarrow \nabla f(\boldsymbol{x}_{t-1}) $
   \STATE $\boldsymbol{m}_t \leftarrow \beta_t \boldsymbol{m}_{t-1}  + \alpha_t  \boldsymbol{g}_t$ 
   \STATE $\boldsymbol{x}_t \leftarrow \boldsymbol{x}_{t-1} - \eta_t \frac{\boldsymbol{m}_t}{|\boldsymbol{m}_t|^q + \xi} $
   \ENDFOR 
   \STATE {\bfseries Output:} $\boldsymbol{x}_T$  
   \\ \hrulefill \\
  \hspace{-5mm} *\hspace{0.1mm} \textbf{Experimental setup}:  $\alpha_t$: learning rate  \newline
  $\textrm{ }\quad\;\;\;\;\;$  $(\beta_t, \eta_t, \xi) =(0.9, 1, 0)$, $q=0.125\textrm{ and }  0.25$
\end{algorithmic}
\end{algorithm}

We firstly introduce notations for mathematical description in the remainder of the paper.  We use bold small letters to denote vectors and bold capital letters to denote matrices. Given a vector $\boldsymbol{x}\in \mathbb{R}^d$, we denote its $l_1$, $l_2$ and $l_{\infty}$ norm as $\|\boldsymbol{x}\|_1 = \sum_{i=1}^d |x_i| $,   $\|\boldsymbol{x}\|_2 = \sqrt{\sum_{i=1}^d x_i^2}$ and $\|\boldsymbol{x}\|_{\infty} = \max_{i=1}^d |x_i |$, respectively. We write the vector obtained by computing the absolute value per coordinate of $\boldsymbol{x}$ as $|\boldsymbol{x}|$. The operation $\textrm{diag}(\boldsymbol{x})$ denotes a diagonal matrix with $\boldsymbol{x}$ on its diagonal. Given two vectors $\boldsymbol{x}, \boldsymbol{y}\in \mathbb{R}^d$, $\boldsymbol{x} \odot \boldsymbol{y}$ and $\boldsymbol{x}/\boldsymbol{y}$ represent element-wise vector multiplication and division, respectively. The operation $\langle \boldsymbol{x} , \boldsymbol{y} \rangle $ denotes the inner product of the two vectors. For a matrix $\boldsymbol{M}\in \mathbb{R}^{d\times d}$, we use $\lambda_{\max}(\boldsymbol{M})$ and $\lambda_{\min}(\boldsymbol{M})$ to denote the largest and smallest singular values of $\boldsymbol{M}$, respectively. 

We attempt to solve the following minimization problem of a finite functional sum 
\begin{align}
\boldsymbol{x}^{*}  = \arg\min_{\boldsymbol{x}\in \mathbb{R}^d} f(\boldsymbol{x}) = \arg\min_{\boldsymbol{x}\in \mathbb{R}^d} \sum_{i=1}^k  f_i(\boldsymbol{x}),
\label{equ:problem}
\end{align}
where the $k$ functions $\{f_i\}_{i=1}^k$ are assumed to be continuously differentiable. In practice, the vector $\boldsymbol{x}$ can be taken as representing the weights of a DNN.  Each function $f_i$ in (\ref{equ:problem}) can be considered to be constructed from a minibatch of training samples. In total, the $k$ functions cover all the training samples.  At each iteration during the optimization procedure, one can either randomly select a function for computation or follow a predefined order from $\{f_i\}_{i=1}^k$.  The above minibatch-based scheme makes it possible to minimize the overall function $f(\boldsymbol{x})$ under the condition of an extremely large number of training samples and limited computational resources in practice.

%Suppose the matrix $\boldsymbol{M}$ is symmetric. 
%The notation $\boldsymbol{M}\succeq 0$ (or $\boldsymbol{M}\succ 0$) represents a symmetric positive semi-definite matrix (or a symmetric positive definite matrix). The superscript $(\cdot)^T$ represents the transpose operator. Given a vector $\boldsymbol{y}$, we use $\|\boldsymbol{y}\|$ to denote its $l_2$ norm. %As an extension, we use $\|\boldsymbol{y}\|_{\boldsymbol{M}}$ to denote the $\boldsymbol{M}$-norm of $\boldsymbol{y}$ where $\boldsymbol{M}\succ 0$, i.e., $\|\boldsymbol{y}\|_{\boldsymbol{M}}=\sqrt{\boldsymbol{y}^T\boldsymbol{My}}$.

\vspace{0mm}
\section{Rapidly Adapting Moment Estimation}
\vspace{-0mm}
\label{sec:RAME}

%In this section, we first revisit Adam by studying its update expressions. We then motivate our new training method RAME. In the end, we provide guidelines how to implement the new method in practice.    

\subsection{On effectiveness of HB}
\vspace{0mm}
\label{subsec:HB}

In this subsection, we first briefly present the empirical results collected in \cite{Jastrzebski19SGD} by analyzing vanilla SGD. We then study the effectiveness of HB by drawing connections between its update expressions and the observations made in \cite{Jastrzebski19SGD}. 
 
The recent work \cite{Jastrzebski19SGD} investigates the performance of vanilla SGD by testing various setups of the learning rates along different curvature directions.  At every iteration, the Hessian matrix is computed in addition to the gradient vector. The sharp curvature directions are then identified as the eigenvectors of the Hessian matrix with large eigenvalues. The model parameters are updated by first projecting the gradient vector along the eigenvectors and then setting individual learning rates along the projections. It is found that faster convergence and better generalization performance can be achieved by setting smaller learning rates for the sharp curvature directions than for the flat directions. That is, it is preferable to suppress the impact of the contributions from the sharp curvature directions and enhance the  impact from the remaining directions.  The above observations are reasonable as sharp curvature directions would lead to high probabilities of missing the local minimums if their learning rates are not set small.  

In practice, it is rather expensive to compute the Hessian matrix. The HB method captures information of the functional curvature by tracking the first moment $\boldsymbol{m_t}$ over iterations. Since $\boldsymbol{m_t}$ is computed as a weighted average of the past gradients, it is natural that the gradient elements having roughly the same directions across iterations, which correspond to flat curvature directions, would be enhanced. In contrast, the gradient elements with varying directions across iterations due to sharp curvatures would be suppressed in the computation of $\boldsymbol{m_t}$.  As a result, when performing the parameter update, HB implicitly sets smaller learning rates for the sharp curvature directions than for the flat curvature directions as suggested by the recent work \cite{Jastrzebski19SGD}. 

%The above analysis suggests that HB implicitly makes use of the functional curvatures for effective parameter-update.
 We note that the effectiveness of HB can be pushed to a higher level in different ways. It is known that the NAG method accelerates HB by constructing a different linear function of the first moment and gradient in the parameter-update. On the other hand, existing adaptive gradient methods such as Adam modify HB by introducing individual learning rates in addition to the common learning rate $\alpha_t$ or $\eta_t$ in (\ref{equ:SGDM1})-(\ref{equ:SGDM2}).  By doing so, these methods receive more algorithmic flexibility than HB, leading to a more effective parameter-update.  In this work, we intend to construct and apply a nonlinear function of the first moment in the parameter-update of HB, as will be discussed later on.

\subsection{Revisiting Adam}

Currently, Adam \cite{Kingma17} is probably the most popular adaptive gradient method in the deep learning community, of which the update expressions can be written as
\begin{align}
\boldsymbol{m}_{t} & =\beta_1\boldsymbol{m}_{t-1} + (1-\beta_1) \nabla f_{t_i}(\boldsymbol{x}_{t-1}) \label{equ:Adam1}\\
\boldsymbol{v}_{t} & = \beta_2 \boldsymbol{v}_{t-1} + (1-\beta_2) |\nabla f_{t_i}(\boldsymbol{x}_{t-1})|^2 \label{equ:Adam2} \\
\boldsymbol{x}_{t} &= \boldsymbol{x}_{t-1} - \alpha_t \frac{\boldsymbol{m}_{t}}{\sqrt{\boldsymbol{v}_{t}}+\xi},  \label{equ:Adam3}
\end{align}
where $0<\beta_1, \beta_2<1$, and $f_{t_i}$ represents the function being selected from the $k$ functions in (\ref{equ:problem}) at iteration $t$. The parameter $\xi > 0$ in  (\ref{equ:Adam3}) is introduced to avoid division by zero. %Recent research by \cite{Adam18Converge} shows that the parameter $\xi$ has a significant impact on the behaviour of Adam. It is not clear yet how to set up the parameter in an optimal manner.  
The parameter $\alpha_t $ is the common learning rate while $1/(\sqrt{\boldsymbol{v}_{t}}+\xi)$ represents the individual learning rates.% computed from the second moment $\boldsymbol{v}_{t}$. 

Equ.~(\ref{equ:Adam2}) indicates that the second moment $\boldsymbol{v}_{t}$ is obtained from the moving average of squared gradients.  That is, only the magnitude information of gradients is reflected in the second moment.  With the computation of $1/(\sqrt{\boldsymbol{v}_{t}}+\xi)$, the gradient elements with large magnitudes across iterations would lead to small learning rates. On the other hand, those with small magnitudes would receive large learning rates and tend to be aggressive when updating their corresponding coordinates of $\boldsymbol{x}$. This allows Adam to adjust the individual learning rates in a  self-adaptive manner.   

Finally, it is clear that the first and second moments of Adam carry different dynamic variations of gradients over iterations.  The first moment takes the sign of gradients into consideration which is missing in the second moment. One natural research question is if the first moment itself can be used for learning-rate computation. Usage of the second moment might not be the only approach to compute the individual learning rates.

\begin{remark}
The method RMSprop \cite{Tieleman12RMSProp} can be taken as a special case of Adam by letting $\beta_1 =0$ in (\ref{equ:Adam1}).  That is, only the second moment is computed for the learning-rate computation.  
\end{remark} 
 
%From a scientific point of view,  it would be of great interest to study the feasibility and effectiveness of developing an adaptive gradient method without tracking the second moment.  The new research direction might alleviate the problem of unsatisfactory generalization performance of existing adaptive gradient methods as reported in \cite{Wilson17AdamNegative}. 

 %the research will reveal novel insights into . 

  %The parameter is $\beta_2$ (e.g., 0.999 as recommended in \cite{Kingma17})  is often set to close to 1. The 
 
\subsection{Algorithm design}
\label{subsec:alg}

Differently from the design strategies of existing adaptive gradient methods, we attempt to make the HB method less aggressive by introducing a nonlinear function of the first moment in the parameter-update.  In particular,  we design the update expressions of the new method RAME to be  
\begin{align}
\boldsymbol{m}_t &= \beta_t \boldsymbol{m}_{t-1}  + \alpha_t \nabla f_{t_i}(\boldsymbol{x}_{t-1}) \label{equ:RAME1} \\
\boldsymbol{x}_t &=  \boldsymbol{x}_{t-1} - \eta_t \boldsymbol{h}(\boldsymbol{m}_t),  \label{equ:RAME2}  
\end{align}
where $(\alpha_t,\eta_t)$ are inherited from (\ref{equ:SGDM1})-(\ref{equ:SGDM2}), and $\boldsymbol{h}(\boldsymbol{m}_t)$ is a $d$-dimensional nonlinear function of $\boldsymbol{m}_t$, given by
\begin{align}
 \boldsymbol{h}(\boldsymbol{m}_t) =  \frac{\boldsymbol{m}_t}{|\boldsymbol{m}_t|^q + \xi},  \label{equ:RAME2_} 
\end{align}
where $\xi \geq 0$ and $1> q\geq 0$.  The upper bound $1> q$ is imposed due to the fact when $q=1$ and  $\xi=0$,  the magnitude of $\boldsymbol{m}_t$ will be cancelled in computing $\boldsymbol{x}_t$, which is undesirable.  The update expressions (\ref{equ:RAME1})-(\ref{equ:RAME2_}) are for minibatch-based DNN training. At each iteration, one individual  function is selected from the total $k$ functions for the parameter update. When the overall $f(\boldsymbol{x})$ is considered per iteration, RAME becomes deterministic, which is summarized in Alg. 1. 

The nonlinear function $\boldsymbol{h}(\boldsymbol{m}_t)$ ensures that the components of $\boldsymbol{m}_t$ with large magnitudes receive smaller learning rates, thus making RAME less aggressive than HB. The motivation behind this modification is that the parameters $\{\beta_t\}$ of HB are usually set to be close to 1 while the $\{\alpha_t\}$ form a decreasing sequence in practice (see \cite{Krizhevsky} for an example). In this situation, the individual learning rates $1/( | \boldsymbol{m}_t|^q+\xi)$ make it easier for $\boldsymbol{m}_t$ to capture the local functional structure around $\boldsymbol{x}_{t-1}$.

Conceptually speaking, RAME utilises the dynamics of gradient information  to compute the individual learning rates while Adam employs the dynamics of gradient-magnitude information. We note that the results of \cite{Jastrzebski19SGD} on the Hessian do not suggest but also do not preclude a relation between the gradient-magnitude information and the optimal individual learning rates. The gradient information may also be a good candidate for  computing the individual learning rates.

%The new method avoids tracking the second moment of gradients and only relies on the most recent first moment. Therefore, it is efficient in memory compared to Adam by saving a memory space of the DNN model size.    

One common property of RAME and Adam (with fixed $\beta_2$ parameter in (\ref{equ:Adam2}))  is that the individual learning rates of both methods do not decrease monotonically over iterations, which makes it challenging for convergence analysis.  In contrast, the three adaptive gradient methods AMSGrad, PAdam and AdaGrad from literature are designed to ensure the property of monotonically decreasing individual learning rates. We note that, at the moment, Adam has gained more popularity than the above three methods for training various DNN models. It might be the non-monotonicity property of the individual learning rates in Adam that makes it remarkably effective. The above hypothesis provides one motivation in designing RAME in this work.

\subsection{Implementation for different setups of $\xi$}
In this subsection, we study the implementation of RAME. Depending on the parameter $\xi$, the computation for $\boldsymbol{x}_t$  can be implemented in different ways. When $\xi>0$, each coordinate of $|\boldsymbol{m}_t|^q + \xi$ in the denominator is nonzero.  In this case,  $\boldsymbol{x}_t$ can be computed in a traditional manner without worrying about zero-division. 

We now consider the setup $\xi=0$. As $\boldsymbol{m}_t$ is obtained by a weighted summation of the past gradients up to iteration $t$, it may happen that certain coordinates of  $|\boldsymbol{m}_t|^q$ are zero. To avoid zero-division, we can simply combine $\boldsymbol{m}_t$  and $|\boldsymbol{m}_t|^q$ in (\ref{equ:RAME2_}) when updating $\boldsymbol{x}_t$. That is,  $\boldsymbol{x}_t$ can be computed as          
\begin{align}
\boldsymbol{x}_t &= \boldsymbol{x}_{t-1} - \eta_t \frac{\boldsymbol{m}_t}{|\boldsymbol{m}_t|^q } \nonumber \\
&= \boldsymbol{x}_{t-1} - \eta_t \cdot \textrm{sign}{(\boldsymbol{m}_t) } \odot  |\boldsymbol{m}_t|^{1-q},  \label{equ:RAME3}
\end{align}
where the operator $\textrm{sign}(\cdot)$ computes the sign of the vector.  % It is immediate that .

It is worth pointing out that Adam and other existing adaptive gradient methods do not allow the special setup $\xi=0$.  This is because the dynamics of the second moment $\boldsymbol{v}_t$ is different from those of $\boldsymbol{m}_t$ or $\boldsymbol{g}_t$. They cannot be combined in a similar manner to (\ref{equ:RAME3}).  

\section{A Different Perspective of the Update Expressions of RAME}
\label{sec:reformulation}

In this section, we study deterministic RAME in Alg.~1 under the setup $\{(\eta_t, \xi)=(1, 0)\}$  from a different perspective. To do so, we first revisit an alternative representation of the update expressions of HB under $\{\eta_t=1\}$.   Based on the observations for HB, we then study the update expressions of RAME from a different point of view.   

\subsection{Revisiting HB under $\{\eta_t=1\}$}
It can be shown that the update expressions (\ref{equ:SGDM1})-(\ref{equ:SGDM2}) of HB under the setup $\{\eta_t=1 \}$ can be alternatively represented as \cite{Polyak64CM, Yang2016HB}
\begin{align}
\boldsymbol{x}_{t+1} - \boldsymbol{x}_{t}   &=  -\alpha_t \boldsymbol{g}_t  + \beta_t  (\boldsymbol{x}_t -  \boldsymbol{x}_{t-1}), \label{equ:HB_alt}
\end{align}
where $\alpha_t>0$ and $0\leq \beta_t <1$. It is clear from (\ref{equ:HB_alt}) that the update of $\boldsymbol{x}_{t+1}$ consists of two contributions: one from the current gradient $ \boldsymbol{g}_t$ and the other from the most recent steering vector $(\boldsymbol{x}_t -  \boldsymbol{x}_{t-1})$. In practice, the parameter $\alpha_t$ decreases over  $t$  while $\{\beta_t\}$ are usually set to be close to 1. Therefore, as the iteration index $t$ increases, the steering vector $(\boldsymbol{x}_t -  \boldsymbol{x}_{t-1})$ has an increasing impact on $\boldsymbol{x}_{t+1}$ compared to the gradient $\boldsymbol{g}_t$. The method name  ``\emph{heavy-ball}" indicates that the update $\boldsymbol{x}_{t+1}$ is strongly affected by the most recent steering vector $(\boldsymbol{x}_t -  \boldsymbol{x}_{t-1})$. %As discussed in Subsection \ref{subsec:HB}, the above strong effect might be too much at a later stage of the iteration process.
  
Algebraically speaking, (\ref{equ:HB_alt}) can be viewed as a dynamic system  describing the evolution of the steering vectors $\{ \boldsymbol{x}_{i+1} - \boldsymbol{x}_{i}| i=0,1, \ldots \}$ over iterations. $\{\beta_t\}$ are the damping scalars penalizing old steering vectors when computing new ones. 

\subsection{Deterministic RAME under $\{(\eta_t, \xi) =(1, 0)\}$ } 
Thus-far we have briefly studied HB from a dynamic system point of view. In this subsection, we reconsider deterministic RAME  also from a dynamic system perspective. To do so,  we set $\{(\eta_t, \xi) =(1, 0)\}$ in Alg.~1 for RAME. 

We first reformulate the update expressions of deterministic RAME in a similar manner as that of HB, which is presented in a proposition below: 
\begin{proposition}
Let $\{(\eta_t, \xi) =(1, 0)\}$ in Alg.~1. The update expressions of deterministic RAME  can then be reformulated as
\begin{align}
&(\boldsymbol{x}_{t+1} -  \boldsymbol{x}_{t})\odot |\boldsymbol{x}_{t+1} -  \boldsymbol{x}_{t}|^{q/(1-q)}   \nonumber \\
&=  -\alpha_t \boldsymbol{g}_t  + \beta_t  (\boldsymbol{x}_t -  \boldsymbol{x}_{t-1})\odot |\boldsymbol{x}_t -  \boldsymbol{x}_{t-1}|^{q/(1-q)}, 
\label{equ:RAME_dyn}
\end{align}
where $0\leq q < 1$, and the iteration index $t\geq 1$.  
\label{prop:RAME_dyn}
\end{proposition}

\begin{proof}
We show that (\ref{equ:RAME_dyn}) can be transformed to  the update expressions presented in Alg.~1 under the setup  $\{(\eta_t, \xi) =(1, 0)\}$. Define $\tilde{\boldsymbol{m}}_{t}$ to be 
\begin{align}
\hspace{-2mm}\tilde{\boldsymbol{m}}_{t} \hspace{-0.5mm}=\hspace{-0.5mm} -(\boldsymbol{x}_{t+1} \hspace{-0.5mm}-\hspace{-0.5mm}  \boldsymbol{x}_{t})\odot |\boldsymbol{x}_{t+1} \hspace{-0.5mm}-\hspace{-0.5mm}  \boldsymbol{x}_{t}|^{q/(1-q)} \quad t\geq 0. 
\label{equ:RAME_dyn1}
\end{align}
It is straightforward from (\ref{equ:RAME_dyn}) that the sequence $\{\tilde{\boldsymbol{m}}_{t} \}$ can be computed recursively as 
\begin{align}
\tilde{\boldsymbol{m}}_{t} = \beta_t \tilde{\boldsymbol{m}}_{t-1} +  \alpha_t \boldsymbol{g}_t  \quad t\geq 1,  \label{equ:RAME_dyn2} 
\end{align}
where the minus sign before $\boldsymbol{g}_t $ in (\ref{equ:RAME_dyn}) is cancelled out due to the minus sign in (\ref{equ:RAME_dyn1}).

Next, without loss of generality, we derive an explicit update expression for $\boldsymbol{x}_{t+1}$ in terms of $\tilde{\boldsymbol{m}}_{t}$ based on (\ref{equ:RAME_dyn1}). %We note that (\ref{equ:RAME_dyn}) can be alternatively represented in terms of $( \boldsymbol{x}_{t+1}, \boldsymbol{x}_{t}, \tilde{\boldsymbol{m}}_{t}) $ as   
%\begin{align}
%(\boldsymbol{x}_{t+1} -  \boldsymbol{x}_{t})\odot |\boldsymbol{x}_{t+1} -  \boldsymbol{x}_{t}|^{q/(1-q)}  =  - \tilde{\boldsymbol{m}}_{t}. \label{equ:RAME_dyn3}
%\end{align}
Taking absolute value per-coordinate on both sides of (\ref{equ:RAME_dyn1}) and then applying algebra produces 
\begin{align}
 |\boldsymbol{x}_{t+1} -  \boldsymbol{x}_{t}| = |\tilde{\boldsymbol{m}}_{t}|^{1-q}. \label{equ:RAME_dyn4}
\end{align}
Finally, plugging (\ref{equ:RAME_dyn4})  into (\ref{equ:RAME_dyn1}) and rearranging the quantities in the equation yields 
\begin{align}
\boldsymbol{x}_{t+1} = \boldsymbol{x}_{t} - \frac{ \tilde{\boldsymbol{m}}_{t}}{ | \tilde{\boldsymbol{m}}_{t}|^{q} }. \label{equ:RAME_dyn5}
\end{align} 

By letting $\{\tilde{\boldsymbol{m}}_{t} = {\boldsymbol{m}}_{t}| t\geq 0\}$, it is immediate that the expressions (\ref{equ:RAME_dyn2}) and (\ref{equ:RAME_dyn5}) are identical to those in Alg.~1 under the setup $\{(\eta_t, \xi) =(1, 0)\}$. The proof is complete. 
\end{proof}

Equ.~(\ref{equ:RAME_dyn}) is a natural extension of (\ref{equ:HB_alt}) for HB. Each steering vector $(\boldsymbol{x}_{t+1} - \boldsymbol{x}_{t})$ in (\ref{equ:RAME_dyn}) is modulated by the $\frac{q}{1-q}$th order of its magnitude, which is represented as $|\boldsymbol{x}_{t+1} -  \boldsymbol{x}_{t}|^{q/(1-q)}$. %Therefore, the modulation is self-adaptive and does not need to store any historical information. 
In the computation of $\boldsymbol{x}_{t+1} $,  the modulation imposes a larger suppression on those elements of  $(\boldsymbol{x}_{t+1} - \boldsymbol{x}_{t})$ with large magnitude than on the remaining elements.  From an overall perspective, (\ref{equ:RAME_dyn}) can be viewed as a dynamic system describing the evolution of the modulated steering vectors  $\{ (\boldsymbol{x}_{i+1} - \boldsymbol{x}_{i})\odot |\boldsymbol{x}_{i+1} -  \boldsymbol{x}_{i}|^{q/(1-q)} | i=0,1, \ldots \}$ over iterations. %When $q=0$, (\ref{equ:RAME_dyn}) becomes (\ref{equ:HB_alt}), indicating that RAME reduces to HB.  

\section{Convergence Analysis for Deterministic RAME}
\label{sec:analysis}

In this section, we provide convergence analysis for employing deterministic RAME to solve $L$-smooth nonconvex optimization. %From a high-level point of view, it is a challenging task.
Similarly to Adam with fixed parameter $\beta_2$, the individual learning rates of RAME $\{\frac{1}{|\boldsymbol{m}_t|^q+\xi} | t=1,2, \ldots\}$ are not guaranteed to decrease monotonically over iterations. Therefore, the approaches in  \cite{Reddi18Amsgrad, Zhou18con_PAdam, Chen18con_Adam} for analyzing  AMSGrad, PAdam and AdaGrad can not be exploited to study either Adam or RAME.  To our best knowledge, the recent work \cite{Adam18Converge} is the first that provides a rigorous convergence analysis for deterministic Adam  for solving $L$-smooth nonconvex optimization.  In the following, we study RAME by following an analysis similar to the one in \cite{Adam18Converge} for Adam.
  
We first provide the definition of $L$-smoothness. 

\begin{definition} [$L$-smoothness] Suppose $f: \mathbb{R}^d \rightarrow \mathbb{R}$ is differentiable. Then $f$ is $L$-smooth for some $L>0$ if for any $\boldsymbol{x}, \boldsymbol{y} \in \mathbb{R}^d$, we have 
\begin{align}
f(\boldsymbol{y}) \leq f(\boldsymbol{x}) + \langle \nabla f(\boldsymbol{x}), \boldsymbol{y}-\boldsymbol{x} \rangle + \frac{L}{2} \|\boldsymbol{y} - \boldsymbol{x} \|^2. 
\label{equ:sm1}
\end{align}
Furthermore, $f(\boldsymbol{x})$ is lower bounded, i.e., $\inf_{\boldsymbol{x}}f(\boldsymbol{x})> - \infty$. %$f(\boldsymbol{x}^{\star}) > - \infty$, where $\boldsymbol{x}^{\star}$ is a global optimal solution to (\ref{equ:NSO}). 
\label{def_smooth}
\end{definition}

Upon introducing $L$-smoothness, we present the convergence results of deterministic RAME in a theorem below: 
\begin{theorem}
Suppose $f: \mathbb{R}^d \rightarrow \mathbb{R} $ is an $L$-smooth function and the $l_{\infty}$ norm of its gradient $\nabla f(\boldsymbol{x})$ is upper bounded by $\|\nabla f(\boldsymbol{x}) \|_{\infty} \leq \sigma$. Let $\xi>0$ and $(\beta_t, \alpha_t) = (\beta, \alpha)$ in Alg. ~1. For any $\epsilon >0$, if the two parameters $(\beta, \alpha)$ are selected to satisfy  
\begin{align}
\beta  &< \frac{\epsilon}{ \sqrt{d} \sigma + \epsilon} \label{equ:beta_bound} \\
\alpha &<  \left( \frac{\xi (1-\beta)^q}{\sigma^q}  \left( \frac{(1-\beta) \epsilon }{\beta \sqrt{d} \sigma} -1\right)\right)^{1/q}, \label{equ:alpha_bound}
\end{align}
then there exist an iteration index $T$ and a sequence of parameters $\{\eta_t>0| t=1, 2, \ldots, T\}$ such that 
\begin{align}
\min_{t=2,\ldots, T+1} \|\nabla f(\boldsymbol{x}_t) \|_2 \leq \epsilon. \nonumber
\end{align}
\label{theorem:converge}
\end{theorem}
\begin{proof}
See proof sketch in Appendix \ref{appendix:RAME_converge}. The basic idea of the argument is from the proof for Theorem 3.4 in \cite{Adam18Converge} for analyzing deterministic Adam. 
\end{proof}

In practice, the parameter $\beta$ is usually set to be constant. The condition (\ref{equ:beta_bound}) is therefore rather strict. It remains open to tighten the convergence analysis to derive a loose condition on $\beta$. In Theorem~\ref{theorem:converge}, $\alpha$ can be treated the learning rate while the parameters $\{\eta_t>0| t=1, 2, \ldots, T\}$ can be taken as the additional regulation parameters for the convergence results to hold.   

\begin{figure*}[t]
\centering
\begin{footnotesize}
  \includegraphics[width=170mm]{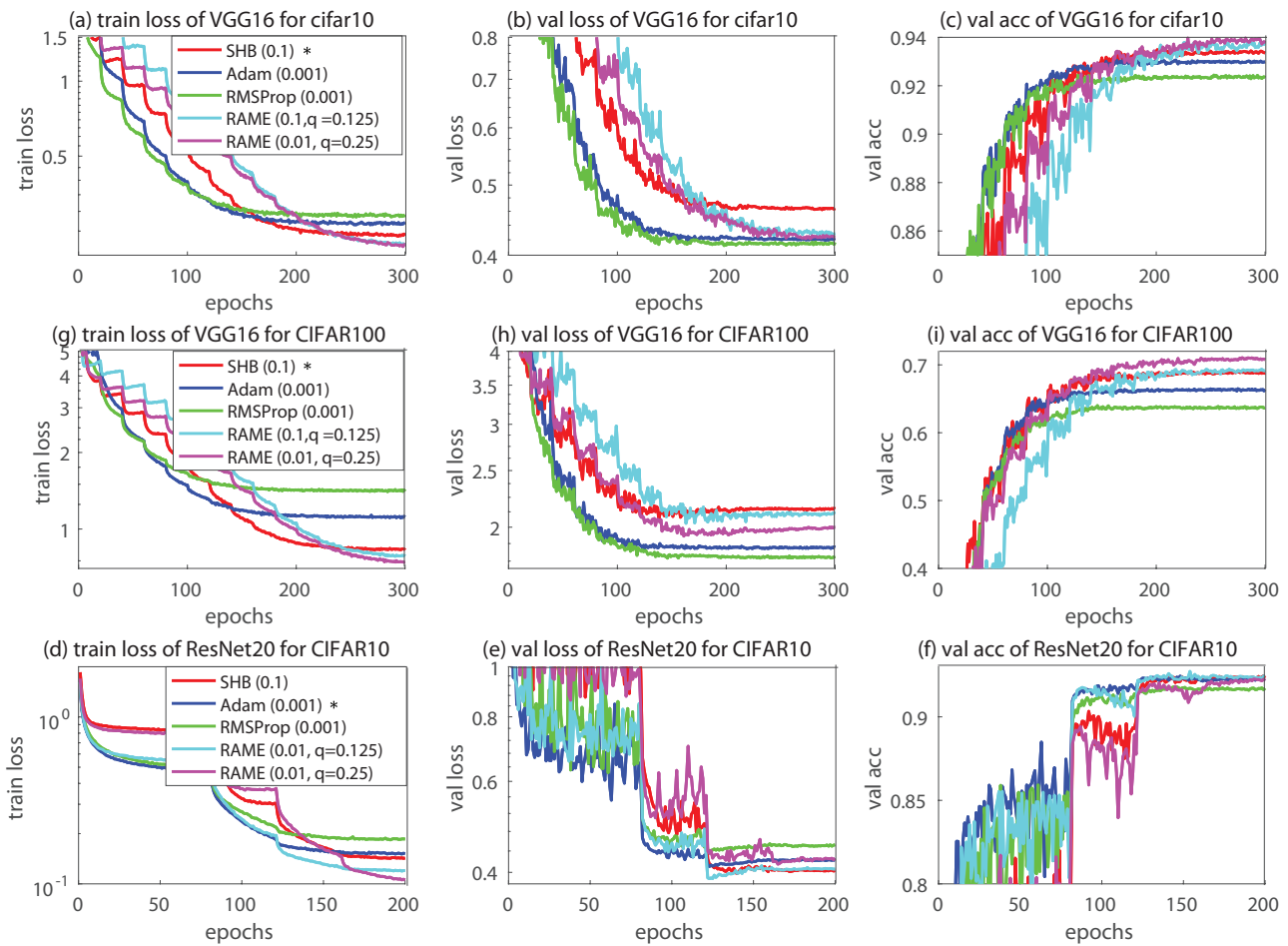}
\end{footnotesize}
\caption{\small Performance comparison of the four methods for training VGG16 over CIFAR10 and CIFAR100, and for training ResNet20 over CIFAR10. The symbol $*$ in the plots indicates that SHB was empoyed in the original open-sources for training VGG16 while Adam was used for ResNet20. The selected initial learning rate for each training method is displayed in the respective bracket. }
\label{fig:vgg_resent_vgg}
\vspace{-6mm}
\end{figure*}

\section{Experimental Results}
\label{sec:experiment}

We conduct experiments for two typical problems in deep learning community, which are classification and segmentation of images. The segmentation problem can be viewed as performing regression as the objective function is a combination of binary cross-entropy and  Dice loss on an image-pixel level \cite{Milletari16VNet}.    

\subsection{Experimental setup}
In the experiment, four training methods were evaluated using the Keras-tensorflow platform, which are RAME, SHB,  Adam, and RMSprop. To make a fair comparison, all the experiments were conducted based on open-source implementations, links of which will be provided for each task later on.  In our implementation, only the training methods and initial learning rates were changed in the original codes for algorithmic comparison. 

%To reduce , Only the training methods were changed in the original codes for algorithmic comparison.  %Furthermore,  the open-sources were selected to preserve algorithmic diversities. In particular, some implementations use SHB while  others employ Adam.  

%The parameters of the training methods were selected by accommodating the setups in the original open-source implementations.  
We now briefly explain the parameter setup for each training method in the experiment. Considering SHB, the parameters $\{\alpha_t\}$ are taken as the common learning rates.  The setup $(\beta_t, \eta_t) = (0.9, 1)$ of SHB (see (\ref{equ:SGDM1})-(\ref{equ:SGDM2})) was inherited from the open-source for training VGG16, which will be studied in Subsection \ref{subsec:VGGResNet} later on.  The parameters ($\beta_1, \beta_2, \xi$) of Adam were set to $(0.9, 0.999,10^{-7} )$, which are the default values of the Keras platform. This is because the open source for ResNet20 (see Subsection \ref{subsec:VGGResNet}) recommends to use Adam with default values.  Similarly, the parameters of RMSprop were set to the default values of the Keras platform. %which have been widely used for training various types of DNNs (see the examples\footnote{https://github.com/keras-team/keras/blob/master/examples} from Keras platform).    

 As RAME is a natural extension of HB (or SHB), its parameters were set to $(\beta_t, \eta_t, \xi) = (0.9, 1, 0)$ and $q = (0.125 \textrm{ and } 0.25)$ as stated in Alg.~1. Our main motivation for choosing $\xi=0$ is because with this setup, deterministic RAME possesses a unified update expression in terms of the modulated steering vectors as summarized in Proposition~\ref{prop:RAME_dyn}. 

Finally, we note that selection of the initial learning rate is essential for the success of a training method.  As different training methods are designed by following respective strategies, their optimal initial learning rates are usually different (see \cite{Chen18Padam} for an empirical study of several training methods). In our experiment, five initial learning rates were tested when employing each method in training a DNN, which are given by $\{10^{-i} | i=1,2\ldots, 5 \}$. Only the convergence result of the initial learning rate that produces the best validation performance was selected for comparison.

%\subsection{On classification}
%In this subsection, we present the experimental results for classification over CIFAR10 and CIFAR100. We evaluated four training methods for performance comparison, which are RAME, SHB,  Adam, and RMSprop. For completeness, two convolutional neural networks (CNNs) and one multiple layer perceptron (MLP) network are investigated. 

\vspace{-3mm}
\subsection{On training VGG16 and ResNet20 over CIFAR10 and CIFAR100} 
\vspace{-1mm}
\label{subsec:VGGResNet}

In the first experiment, we consider training VGG16 \cite{Simonyan16DCNN}  and ResNet20 \cite{He15ResNet}, which represent two popular convolutional neural network (CNN)  architectures in deep learning. We adopt the existing open sources\footnote{\hspace{0mm} The code for VGG16 is from https://github.com/geifmany/cifar-vgg \newline The code for ResNet20 is adopted from https://github.com/keras-team/keras/blob/master/examples/cifar10$\_$resnet.py} for three tasks, which are training VGG16 over CIFAR10 and CIFAR100, and training ResNet20 over CIFAR10. %In their original implementations, both networks employ the learning-rate scheduling as functions of the epoch by following certain downstair styles. In this work, we keep the learning-rate scheduling and only change the training methods for algorithmic comparison.  sfdfsd

We notice that the original implementation for VGG16 employs SHB while the one for ResNet20 uses Adam. The above open-sources were selected on purpose to minimize algorithmic bias that favours the original training method.  
  
The convergence behaviours of the four methods are displayed in Fig.~\ref{fig:vgg_resent_vgg}.  It is seen that the initial learning rate of SHB is the largest, followed by those of RAME for $q = 0.125 \textrm{ and } 0.25$. If we treat SHB as a special case of RAME with $q=0$, it is clear that as the parameter $q$ increases from $0$ to $0.125$ and finally to $0.25$,  the best initial learning rate decreases accordingly. This might be because as $q$ increases, the individual learning rates $\{\frac{1}{|\boldsymbol{m}_t|^q}\}$ may have increasing impact on the parameter update, thus only requiring a decreasing contribution from the common learning rates $\{\alpha_t\}$.  The above observations  drawn from RAME are in line with the fact that both Adam and RMSprop have the same smallest initial learning rate. %We notice from  Fig.~\ref{fig:vgg_resent_vgg} that  both Adam and RMSprop have the same smallest initial learning rate, which is consistent with the above analysis for RAME. 

%I think this argument could be formulated better: it does not hammer down what is important. The decision is one-of-a-set and the true objective function is binary (false true). A continuous  objective function is an approximate surrogate that is differentiable. The variation in regions of given decision and ground truth are not important.

It is observed from Fig.~\ref{fig:vgg_resent_vgg} that the validation losses and accuracies of Adam and AMSprop are not consistent for VGG16 compared to those of SHB and RAME. That is, both methods produce low validation losses, while their validation accuracies are not high. The true objective function for classification is binary, representing correct or incorrect recognition decisions over one-of-a-discrete-set. To facilitate the training procedure,  a continuous objective function in the form of cross-entropy is introduced as an approximate surrogate. The variation of the functional loss in regions of given decision and ground truth are not important.
%This might be because the easily-recognisable images from the validation dataset receive very low losses. Therefore, even though other challenging images have high losses due to incorrect classification, the overall average validation loss is still low and the corresponding validation accuracy is high. 
Therefore, validation accuracy can be seen as a more reliable measurement than the validation loss when considering a classification problem. %This is because classification accuracy focuses on the percentage of correctly recognised images rather than the losses introduced by the images. 

By inspection of  the training losses and validation accuracies of the four methods in Fig. \ref{fig:vgg_resent_vgg}, we can conclude that RAME outperforms the other three training methods for VGG16 at the end of the training procedure even though it converges slowly in the beginning. Furthermore, as the parameter $q$ increases from $0.125$ to $0.25$, RAME delivers decreasing final training loss and increasing final validation accuracy. Considering ResNet20, it is seen that RAME again yields low final training losses compared to the other three methods.  As for the final validation accuracies, it performs equally well as SHB and Adam.

\begin{figure}[t]
\centering
\begin{footnotesize}
  \includegraphics[width=75mm]{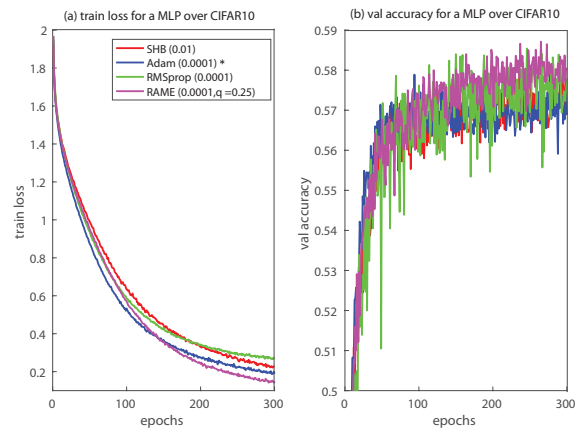}
\end{footnotesize}
\caption{\small Performance comparison of the four methods for training an MLP over CIFAR10. The symbol $*$ indicates that Adam was empoyed in the original open-source.  The selected initial learning rate for each training method is displayed in the respective bracket.  }
\label{fig:MLP}
\vspace{-6mm}
\end{figure}

\begin{figure*}[t]
\centering
\begin{footnotesize}
  \includegraphics[width=180mm]{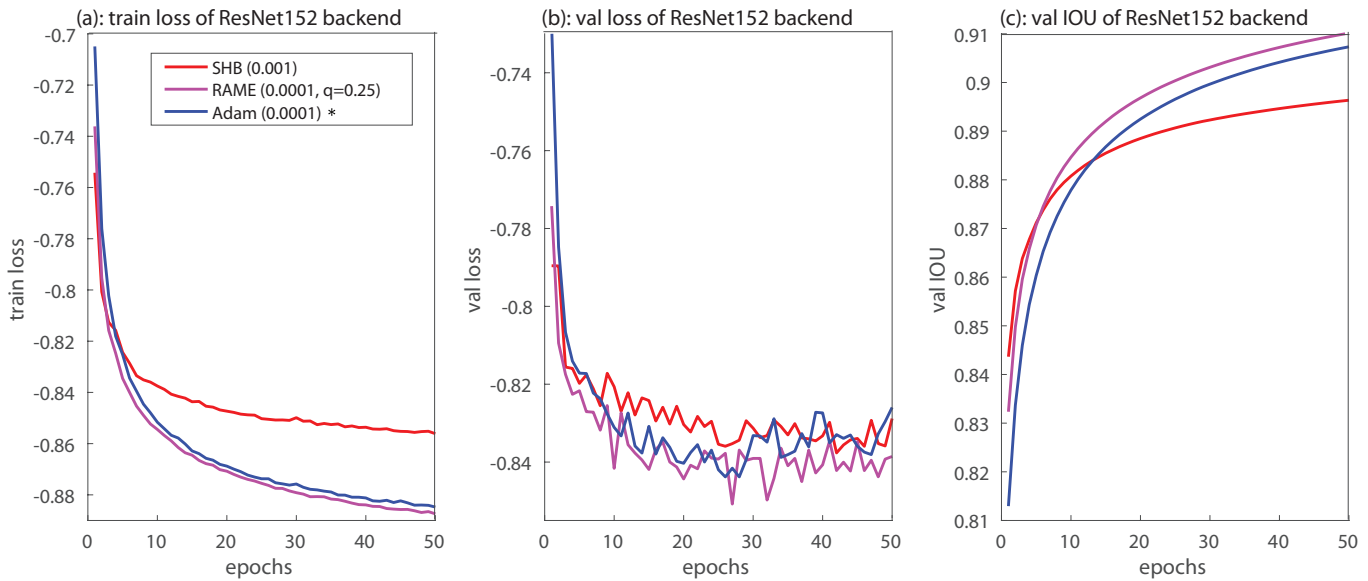}
\end{footnotesize}
\caption{\small On fine-tunning ResNet152 backend for people semantic segmentation.  The symbol $*$ indicates that Adam was employed in the original open-source. The selected initial learning rate for each training method is displayed in the respective bracket.  }
\label{fig:coco}
\vspace{-5mm}
\end{figure*}

\vspace{-1mm}
\subsection{On training a multi-layer perceptron (MLP)} 
\vspace{-1mm}
 In addition to CNNs, we also tested an MLP for classification over CIFAR10, which is in fact a feedforward fully connected neural network. Our primary research goal is to study the convergence behaviours of the four methods for training an MLP rather than producing high validation accuracy. The implementation is based on the open source\footnote{The code for MLP is from https://github.com/aidiary/keras-examples/blob/master/mlp/cifar10.py} available on the Keras platform where Adam with default setup is recommended in the original implementation. The tested MLP consists of four layers with neural numbers of $(1024-512-512-10)$. %The parameters of the four  methods were set in the same way as for training the two CNNs. For each training method, five initial learning rates $\{10^{-i} | i=1,2\ldots, 5 \}$ were tested and only the best learning rate is selected for performance comparison. 

Fig.~\ref{fig:MLP} displays the convergence results of the four training methods. It is seen that RAME converges slower than Adam in the beginning. After a certain number of iterations, it  converges faster than the other three methods and  produces the lowest final training loss.  As for the validation accuracies, the performance of RAME and RMSprop are similar. Both methods produce slightly higher accuracies than Adam and SHB.

\subsection{On semantic segmentation}

We also conduct algorithmic comparison for people semantic segmentation, where the goal is to identify all people in an image on a pixel level  \cite{Long15segmentation}.  To facilitate the training procedure and achieve high accuracy, one approach is to make use of a well-trained neural network for other purposes as the backend for semantic segmentation. In this work, we choose the version of ResNet152 \cite{Lin2016Pyramid, He15ResNet} that is trained for classification over ImageNet as the backend. We adopt an open source implementation developed for a Kaggle competition \footnote{The link is https://github.com/selimsef/dsb2018$\_$topcoders} for our experiment, where Adam with default parameter setup was used for training the network. %Only the top layer is modified for the purpose of pixel-level segmentation. 
As the the main body of the network already carries informative features of 1000 objects in ImageNet database, we only need to fine-tune the network for the segmentation task.  

In the experiment,  the Microsoft COCO-2017 database \cite{COCO}  was employed for training the network. The numbers of images for training and validation are 108344 and 4614, respectively.  Roughly half the number of images in both the training and validation sets contains persons.  
%The images with people were treated as positive images while those without people were taken as negative images.  As we are only interested in people segmentation, only a subset of images were extracted from the overall dataset for our purpose. The numbers of images for training and validation are 108344 and 4614, respectively. 
 
We focus on the performance of SHB, Adam, and RAME (The method RMSprop suffers from significant overfitting effect, and the result is left out to avoid distraction). %The algorithmic parameters and the procedure of initial-learning-rate selection follow directly from the classification problem. The decay of the learning rate was set to $10^{-4}$.  
Each method was fine-tuned for 50 epochs.  Further, each epoch took about one and a half hour using a Nvidia 1080 Ti GPU. During the training process, the so-called \emph{Intersection over Union} (IOU) was measured along with the functional loss. The metric IOU reflects the accuracy of the correctly labelled foreground pixels of people in an image on average.    

The convergence results of the three training methods are displayed in Fig. \ref{fig:coco}.  It is clear that RAME outperforms both SHB and Adam w.r.t. the training loss, validation loss and IOU. On the other hand,  the training loss of SHB is noticeably higher than those of Adam and RAME. This suggests that the introduction of individual learning rates in SHB accelerates the convergence speed.

\subsection{Overall observations from the experiments}

All the above experiments indicate that RAME converges faster than SHB at a later stage of the training procedure. Furthermore, RAME exhibits promising generalisation performance over the validation datasets compared to SHB. The results confirm that it is indeed beneficial to choose smaller learning rates for the elements of $\boldsymbol{m}_t$ with large magnitudes in SHB. %From a high level point of view, the elements of $\boldsymbol{m}_t$ with large magnitudes  can be conceptually connected to the sharp curvature directions as studied in  \cite{Jastrzebski19SGD}. 

% The above empirical evidences are consistent with the observations obtained in \cite{Jastrzebski19SGD} for vanilla SGD (see Subsection \ref{subsec:alg}).
%indicate that the first moments $\{\boldsymbol{m}_t\}$ of gradients can be successfully employed for developing an adaptive gradient method without using the second moment of gradients. 

If we take into account the fact that RAME is designed by making a minor modification to SHB, the new method is both simple and effective. Unlike Adam and RMSprop, RAME does not need to track the second moment of gradients. Instead, the new method only uses a nonlinear function $\boldsymbol{h}(\boldsymbol{m}_t)$ of the most recent first moment for parameter update. %The convergence results show .  
%As demonstrated in Fig.~\ref{fig:vgg_resent_vgg}, the introduced individual learning rates $\{\frac{1}{|\boldsymbol{m}_t|^q}\}$ when $q=0.125 \textrm{ or } 0.25$ allow RAME to achieve lower training losses than SHB which is in fact RAME with $q=0$.     

\vspace{-2mm}
\section{Conclusions}
\label{sec:conclusion}
In this paper, we have proposed a new adaptive gradient method for training DNNs, which is referred to as rapidly adapting moment estimation (RAME). The new method is designed by computing the individual learning rates based on the most recent first moment of gradients rather than the traditional second moment of gradients. Compared to the popular training method Adam, RAME saves a memory space of the DNN model size by avoiding the storage of the second moment. One nice property of RAME is that its update expression can be interpreted as describing the evolution of the modulated steering vectors $\{ (\boldsymbol{x}_{i+1} - \boldsymbol{x}_{i})\odot |\boldsymbol{x}_{i+1} -  \boldsymbol{x}_{i}|^{q/(1-q)} | i=0,1, \ldots \}$ over iterations while other adaptive gradient methods do not have such a property to our best knowledge.  Experimental results for training a number DNNs models demonstrate that RAME produces promising convergence performance in comparison to SHB, Adam, and RMSprop.

One future research direction would be to study the possibility of combining RAME and classical adaptive gradient methods such as Adam for designing a more effective training method. %The research problem boils down to how to make use of both the first and second moments of gradients in computing the individual learning rates.  
 
 \vspace{-2mm}
\section{Acknowledgements}

The research is financially supported by Nippon Telegraph and Telephone (NTT) Corporation, Japan, in a form of an industrial project between UTS and NTT.  

We gratefully acknowledge the assistance offered by Dr Haopeng Li (haopeng.li@qamcom.se) from Qamcom Research and Technology, Sweden, on implementation of the people segmentation experiment.

\appendices

\vspace{-2mm}
\section{ Proof Sketch for Theorem \ref{theorem:converge}}
\label{appendix:RAME_converge}

We study the convergence of Alg.~1 by following a similar argument in  \cite{Adam18Converge} for deterministic Adam. %In principle, one can independently derive similar convergence results as presented in Theorem \ref{theorem:converge} for Alg.~1 by referring to Appendix A.4 in \cite{Adam18Converge}.  We provide a proof sketch here for completeness.  %Similarly to \cite{Adam18Converge}, we will use a contradiction argument to analyze Alg.~1. 
That is we will start from the assumption $\{\|\boldsymbol{g}_t\|>\epsilon |t\geq 1  \} $ and then show that it will lead to a contradiction. 

By following the analysis in \cite{Adam18Converge}, the first step is to find the optimal parameter $\eta_t^{*}$ that leads to a tight upper bound for the functional difference $f(\boldsymbol{x}_{t+1}) - f(\boldsymbol{x}_{t})$.  By using the inequality (\ref{equ:sm1}) due to $L$-smoothness, the functional difference at iteration $t$ can be upper bounded as  
\begin{align}
&f(\boldsymbol{x}_{t+1}) - f(\boldsymbol{x}_{t}) \nonumber \\
&\leq \langle \nabla f(\boldsymbol{x}_t), \boldsymbol{x}_{t+1} - \boldsymbol{x}_t \rangle + \frac{L}{2} \|\boldsymbol{x}_{t+1} - \boldsymbol{x}_t \|_2^2 \nonumber \\  
&\stackrel{(a)}{=} \hspace{-0.6mm} \eta_t \left(-  \hspace{-0.6mm} \left\langle \boldsymbol{g}_t, \frac{\boldsymbol{m}_t}{|\boldsymbol{m}_t|^q+ \xi} \right\rangle  \hspace{-0.6mm}+ \hspace{-0.6mm} \frac{L \eta_t}{2} \left\|\frac{\boldsymbol{m}_t  }{|\boldsymbol{m}_t |^q  \hspace{-0.3mm}+ \hspace{-0.3mm} \xi } \right\|_2^2 \right) \hspace{-0.6mm},\hspace{-2mm} \label{equ:conv_1}
\end{align}
where step $(a)$ follows from the update expression of $\boldsymbol{x}_{t+1}$ in Alg.~1.  It is noted that the RHS of (\ref{equ:conv_1}) is a quadratic function of $\eta_t$. It can be shown that when
\begin{align}
\eta_t=\eta_t^{*} = \frac{1}{2} \cdot \frac{\left\langle \boldsymbol{g}_t, \frac{\boldsymbol{m}_t}{|\boldsymbol{m}_t|^q + \xi} \right\rangle }{\frac{L}{2} \|  \frac{\boldsymbol{m}_t}{|\boldsymbol{m}_t|^q + \xi}   \|^2 }, \label{equ:conv_2}
\end{align}
the LHS of (\ref{equ:conv_1}) receives a tight upper bound, which is given by 
\begin{align}
f(\boldsymbol{x}_{t+1}) - f(\boldsymbol{x}_{t}) \leq -\frac{1}{2L} \cdot \frac{\left\langle \boldsymbol{g}_t, \frac{\boldsymbol{m}_t}{|\boldsymbol{m}_t|^q + \xi} \right\rangle^2}{ \left\| {\frac{\boldsymbol{m}_t}{|\boldsymbol{m}_t|^q + \xi}}  \right \|_2^2 },  \label{equ:conv_3}
\end{align}
which indicates that the functional cost $f(x_t)$ decreases over iteration $t$. 

As described in \cite{Adam18Converge}, the next step is to measure how close it is  between the upper bound in (\ref{equ:conv_3}) and zero. To do so, it is required to derive an upper bound for $\left\| {\frac{\boldsymbol{m}_t}{|\boldsymbol{m}_t|^q + \xi}}  \right \|_2$ and a lower bound for $\left\langle \boldsymbol{g}_t, \frac{\boldsymbol{m}_t}{|\boldsymbol{m}_t|^q + \xi} \right\rangle$, respectively. 

We first present two lemmas which will be used for analysis later on: 
\begin{lemma} Under the setup $(\beta_t, \alpha_t)=(\beta,\alpha)$, the first moment $\boldsymbol{m}_t$ in Alg.~1 can be represented in terms of  $\{\boldsymbol{g}_k | k = 1,\ldots, t\}$ as
\begin{align}
\boldsymbol{m}_t = \alpha \sum_{k=1}^t \beta^{t-k} \boldsymbol{g}_k. \label{equ:conv_4}
\end{align}
Correspondingly, under the assumption $\|\boldsymbol{g}_k \|_{\infty} \leq \sigma $ for all $k\geq 1$, the $l_{\infty}$ norm of $\boldsymbol{m}_t$ is upper bounded as
\begin{align}
\|\boldsymbol{m}_t \|_{\infty} \leq \frac{\alpha \sigma (1-\beta^t)}{1-\beta}.  \label{equ:conv_5}
\end{align}
\end{lemma}

\begin{lemma}
The minimum and maximum eigenvalue of the diagonal matrix $\textrm{diag}(|\boldsymbol{m}_t|^q + \xi)$ satisfy  
\begin{align}
 &\lambda_{max}(\textrm{diag}(|\boldsymbol{m}_t|^q + \xi))  \leq  \frac{\alpha^q  \sigma^q(1-\beta^t)^q}{ (1-\beta)^q }+\xi \label{equ:conv_6} 
\end{align}
\begin{align}  
 & \lambda_{min}(\textrm{diag}(|\boldsymbol{m}_t|^q + \xi))  \geq \xi.  \qquad \qquad  \label{equ:conv_7}
\end{align}
\end{lemma}

We now consider deriving an upper bound for $ \left\| {\frac{\boldsymbol{m}_t}{|\boldsymbol{m}_t|^q + \xi}}  \right \|_2$. It is straightforward that  
\begin{align}
 \left\| {\frac{\boldsymbol{m}_t}{|\boldsymbol{m}_t|^q + \xi}}  \right \|_2 &\leq  \left\| {\frac{\boldsymbol{m}_t}{|\boldsymbol{m}_t|^q }}  \right \|_2 \nonumber \\
 & \leq  \left\| {\frac{\boldsymbol{m}_t}{|\boldsymbol{m}_t|^q}}  \right \|_1 \nonumber \\
 &  \leq  \left\| |\boldsymbol{m}_t|^{1-q}  \right \|_1 \nonumber \\    
 & \leq \frac{d \alpha^{1-q} \sigma^{1-q} (1-\beta^t)^{1-q}}{(1-\beta)^{1-q} }, \label{equ:conv_8}
\end{align}
where the last inequality follows from (\ref{equ:conv_5}). 

Inspired by the corresponding analysis in  \cite{Adam18Converge}, the lower bound for $\left\langle \boldsymbol{g}_t, \frac{\boldsymbol{m}_t}{|\boldsymbol{m}_t|^q + \xi} \right\rangle$ can be derived as 
\begin{align}
&\left\langle \boldsymbol{g}_t, \frac{\boldsymbol{m}_t}{|\boldsymbol{m}_t|^q + \xi} \right\rangle \nonumber \\
& \stackrel{(a)}{=}  \left\langle \boldsymbol{g}_t,  \alpha \sum_{k=1}^t  \frac{\beta^{t-k}\boldsymbol{g}_k}{|\boldsymbol{m}_t|^q + \xi} \right\rangle   \nonumber \\
& \stackrel{(b)}{\geq}  \alpha \|\boldsymbol{g}_t \|_2^2 / \lambda_{max}(\textrm{diag}(|\boldsymbol{m}_t|^q + \xi)) \nonumber \\
&  - \alpha \|\boldsymbol{g}_t\|_2 \sqrt{d} \sigma  \cdot \left(\sum_{j=1}^{t-1} \beta^j\right) / \lambda_{min}(\textrm{diag}(|\boldsymbol{m}_t|^q + \xi))  \nonumber \\
& =  \alpha \|\boldsymbol{g}_t \|_2^2 / \lambda_{max}(\textrm{diag}(|\boldsymbol{m}_t|^q + \xi)) \nonumber \\
&  \hspace{5mm}- \frac{\alpha(\beta-\beta^t) }{1-\beta} \|\boldsymbol{g}_t\|_2 \sqrt{d} \sigma   / \lambda_{min}(\textrm{diag}(|\boldsymbol{m}_t|^q + \xi))  \nonumber  \\
& \stackrel{(c)}{\geq}  \alpha \|\boldsymbol{g}_t \|_2^2 / \left(  \frac{\alpha^q \sigma^q (1-\beta^t)^q}{ (1-\beta)^q } +\xi  \right) \nonumber \\
& \hspace{5mm} - \frac{\alpha(\beta-\beta^t) }{(1-\beta)\xi} \|\boldsymbol{g}_t\|_2 \sqrt{d} \sigma \nonumber  \\
& =  \alpha \|\boldsymbol{g}_t \|_2^2   \Bigg( \frac{(1-\beta)^q}{\alpha^q \sigma^q (1-\beta^t)^q+\xi  (1-\beta)^q} \nonumber \\
& \hspace{20mm} - \frac{(\beta-\beta^t)  \sqrt{d} \sigma}{(1-\beta)\xi \| \boldsymbol{g}_t\|_2}  \Bigg) \nonumber \\
& = \frac{ \alpha \|\boldsymbol{g}_t \|_2^2}{\left[\alpha^q \sigma^q (1-\beta^t)^q +\xi  (1-\beta)^q \right] (1-\beta)\xi \| \boldsymbol{g}_t\|_2 }  \nonumber  \\
& \hspace{3mm} \cdot  \Big( (1-\beta)^{1+q}\xi \| \boldsymbol{g}_t\|_2 \nonumber \\
& \hspace{5mm}- \hspace{-0.5mm}(\beta  \hspace{-0.5mm}- \hspace{-0.5mm} \beta^t)  \sqrt{d} \sigma  \left[\alpha^q \sigma^q (1  \hspace{-0.5mm}- \hspace{-0.5mm} \beta^t)^q  \hspace{-0.5mm}+ \hspace{-0.5mm} \xi  (1  \hspace{-0.5mm}- \hspace{-0.5mm} \beta)^q \right]  \Big) \nonumber 
%& = \frac{ \alpha \|\boldsymbol{g}_t \|_2^2 (\beta-\beta^t)  \sqrt{d} \sigma}{\left[\alpha^q  \sigma^q (1-\beta^t)^q+\xi  (1-\beta)^q \right] (1-\beta)\xi \| \boldsymbol{g}_t\|_2 }  \nonumber  \\
%& \hspace{3mm} \cdot  \Bigg( \frac{(1-\beta)^{1+q}\xi \| \boldsymbol{g}_t\|_2}{(\beta-\beta^t) \sqrt{d} \sigma} -  \left[\alpha^q (1-\beta^t)^q \sigma^q+\xi  (1-\beta)^q \right]  \Bigg) \nonumber \\
\end{align}
\begin{align}
& = \frac{ \alpha \|\boldsymbol{g}_t \|_2^2 (\beta-\beta^t) (1-\beta)^q  \sqrt{d} \sigma }{\left[\alpha^q \sigma^q (1-\beta^t)^q +\xi  (1-\beta)^q \right] (1-\beta)\xi \| \boldsymbol{g}_t\|_2  }  \nonumber  \\
& \hspace{3mm} \cdot  \left( \frac{(1-\beta)\xi \| \boldsymbol{g}_t\|_2}{(\beta-\beta^t) \sqrt{d} \sigma} -  \left[\alpha^q \sigma^q \frac{(1-\beta^t)^q}{(1-\beta)^q}+\xi   \right]  \right) \nonumber \\
& = \frac{ \alpha \|\boldsymbol{g}_t \|_2^2 (\beta-\beta^t) (1-\beta)^q  \sqrt{d} \sigma }{\left[\alpha^q  \sigma^q (1-\beta^t)^q+\xi  (1-\beta)^q \right] (1-\beta)\xi \| \boldsymbol{g}_t\|_2  }  \nonumber \\
& \hspace{3mm} \cdot  \left( \xi \left( \frac{(1-\beta) \| \boldsymbol{g}_t\|_2}{(\beta-\beta^t) \sqrt{d} \sigma} -1\right) -  \alpha^q \sigma^q \frac{(1-\beta^t)^q}{(1-\beta)^q}  \right) \nonumber  \\
& = \frac{ \alpha \|\boldsymbol{g}_t \|_2^2 (\beta-\beta^t) (1-\beta)^q  \sqrt{d} \sigma  \left( \frac{(1-\beta) \| \boldsymbol{g}_t\|_2}{(\beta-\beta^t) \sqrt{d} \sigma} -1\right)}{\left[\alpha^q  \sigma^q (1-\beta^t)^q+\xi  (1-\beta)^q \right] (1-\beta)\xi \| \boldsymbol{g}_t\|_2  }  \nonumber  \\
& \hspace{3mm} \cdot    \left( \xi -  \frac{ \alpha^q \sigma^q (1-\beta^t)^q}{(1-\beta)^q  \left( \frac{(1-\beta) \| \boldsymbol{g}_t\|_2}{(\beta-\beta^t) \sqrt{d} \sigma} -1\right)}  \right),    \label{equ:conv_9}
\end{align}
where $(a)$ follows from (\ref{equ:conv_6}), $(b)$ uses the triangle inequality, $(c)$ follows from (\ref{equ:conv_6})-(\ref{equ:conv_7}).  

Now we are in a position to find the support regions for $\beta$ and $\alpha$ such that the lower bound in (\ref{equ:conv_9}) is positive, which is crucial to ensure that $\eta_t^{*}$ in (\ref{equ:conv_2}) is positive. Similarly to the work \cite{Adam18Converge}, we first consider the support region for $\beta$. It is clear from (\ref{equ:conv_9}) that $\beta$ should be chosen such that 
\begin{align}
 \frac{(1-\beta) \| \boldsymbol{g}_t\|_2}{(\beta-\beta^t) \sqrt{d} \sigma} > \frac{(1-\beta) \epsilon }{\beta \sqrt{d} \sigma} > 1,  \label{equ:conv_10}
\end{align}
where the assumption $ \| \boldsymbol{g}_t\|_2 > \epsilon $ is exploited.  Rearranging the inequality (\ref{equ:conv_10}) produces an upper bound for $\beta$:
\begin{align}
\beta  < \frac{\epsilon}{ \sqrt{d} \sigma + \epsilon}.  \label{equ:conv_11}
\end{align}
To simplify analysis later on, a scalar parameter $\theta_1$ is introduced as follows 
\begin{align}
\theta_1=  \frac{(1-\beta) \epsilon }{\beta \sqrt{d} \sigma} - 1 >0. 
\label{equ:theta_1}
\end{align}
Suppose $\beta$ satisfies the condition (\ref{equ:conv_11}). The parameter $\alpha$ should be selected such that 
\begin{align}
 &\xi -  \frac{ \alpha^q \sigma^q (1-\beta^t)^q}{(1-\beta)^q  \left( \frac{(1-\beta) \| \boldsymbol{g}_t\|_2}{(\beta-\beta^t) \sqrt{d} \sigma} -1\right)} \nonumber \\
 &\geq \xi -  \frac{ \alpha^q \sigma^q   }{ (1-\beta)^q  \left( \frac{(1-\beta) \epsilon }{\beta \sqrt{d} \sigma} -1\right)} > 0. \nonumber
\end{align}
Based on the above inequality, an upper bound for $\alpha$ can be derived as
\begin{align}
\alpha  < \left( \frac{\xi (1-\beta)^q}{\sigma^q}  \left( \frac{(1-\beta) \epsilon }{\beta \sqrt{d} \sigma} -1\right)\right)^{1/q}. \label{equ:conv_12}
\end{align}
Similarly to $\theta_1$, a new parameter $\theta_2$ can be introduced as follows 
\begin{align}
\theta_2 = \xi -  \frac{\alpha^q \sigma^q }{(1-\beta)^q\theta_1}  < \xi - 
 \frac{ \alpha^q \sigma^q (1-\beta^t)^q}{(1-\beta)^q  \left( \frac{(1-\beta) \| \boldsymbol{g}_t\|_2}{(\beta-\beta^t) \sqrt{d} \sigma} -1\right)}. \nonumber 
\end{align}
%\begin{align}
 %\frac{ \alpha^q \sigma^q (1-\beta^t)^q}{(1-\beta)^q  \left( \frac{(1-\beta) \| \boldsymbol{g}_t\|_2}{(\beta-\beta^t) \sqrt{d} \sigma} -1\right)} < \frac{\alpha^q \sigma^q }{(1-\beta)^q\theta_1} = \xi - \theta_2
%\end{align}
A similar definition of $\theta_1$ and $\theta_2$ can also be found in \cite{Adam18Converge} for deterministic Adam. Finally, under the two conditions (\ref{equ:conv_11}) and (\ref{equ:conv_12}), the lower bound (\ref{equ:conv_9}) can be simplified as 
\begin{align}
&\left\langle \boldsymbol{g}_t, \frac{\boldsymbol{m}_t}{|\boldsymbol{m}_t|^q + \xi} \right\rangle  \geq \frac{ \alpha \|\boldsymbol{g}_t \|_2 \beta (1-\beta)^q  \sqrt{d} \sigma \theta_1 \theta_2 }{\left[\alpha^q \sigma^q+\xi  (1-\beta)^q \right] \xi }.
\label{equ:conv_13}  
\end{align}

Upon deriving the upper and lower bounds (\ref{equ:conv_8}) and (\ref{equ:conv_13}), the final upper bound for the functional difference in (\ref{equ:conv_3}) can be represented as 
\begin{align}
&f(\boldsymbol{x}_{t+1} - f(\boldsymbol{x}_t)) \nonumber \\
&\leq -\frac{1}{2L} \frac{\left( \frac{ \alpha \|\boldsymbol{g}_t \|_2 \beta (1-\beta)^q  \sqrt{d} \sigma \theta_1 \theta_2 }{\left[\alpha^q \sigma^q+\xi  (1-\beta)^q \right] \xi } \right)^2  }{ \left(\frac{d \alpha^{1-q}  \sigma^{1-q} (1-\beta^t)^{1-q}}{(1-\beta)^{1-q} } \right)^2 } \nonumber \\
& = \hspace{-0.5mm} - \hspace{-0.5mm}\frac{\|\boldsymbol{g}_t \|_2^2}{2L} \frac{\left(  \alpha  \beta (1-\beta)^q  \sqrt{d} \sigma \theta_1 \theta_2  \right)^2  (1-\beta)^{2(1-q)} }{ \left(d (\alpha\sigma)^{1-q} (1 \hspace{-0.7mm}-\hspace{-0.7mm} \beta^t)^{1-q} \right)^2 \hspace{-0.5mm} \left[\alpha^q \sigma^q \hspace{-0.6mm}+\hspace{-0.6mm} \xi  (1 \hspace{-0.7mm} - \hspace{-0.7mm} \beta)^q \right]^2 \hspace{-0.5mm} \xi^2} \nonumber \\
& <  -\frac{\|\boldsymbol{g}_t \|_2^2}{2L} \frac{\left(  \alpha^q\sigma^q   \beta (1-\beta) \theta_1 \theta_2  \right)^2 }{ d  \left[\alpha^q \sigma^q+\xi  (1-\beta)^q \right]^2 \xi^2}.  \hspace{-2mm} \label{equ:conv_14}  
\end{align}

Summing (\ref{equ:conv_14}) from $t= 2$ until $t=T+1$ produces 
\begin{align}
&f(\boldsymbol{x}_{2}) - f(\boldsymbol{x}_{T+2})  \nonumber \\
&= \sum_{t=2}^{T+1} f(\boldsymbol{x}_{t}) - f(\boldsymbol{x}_{t+1}) \nonumber \\
&\geq  \sum_{t=2}^{T+1} \left(\frac{\left(  \alpha^q\sigma^q   \beta (1-\beta) \theta_1 \theta_2  \right)^2 }{ 2Ld \left[\alpha^q \sigma^q+\xi  (1-\beta)^q \right]^2 \xi^2} \right) \| \nabla f(\boldsymbol{x}_t) \|_2^2. \nonumber
\end{align}
As a result, we have 
\begin{align}
\min_{t=2,\ldots, T+1} \|f(\boldsymbol{x}_t)  \|^2 \leq & \frac{ 2Ld \left[\alpha^q \sigma^q+\xi  (1-\beta)^q \right]^2 \xi^2}{T \left(  \alpha^q\sigma^q   \beta (1-\beta) \theta_1 \theta_2  \right)^2} \nonumber \\
&  \cdot \left[ f(\boldsymbol{x}_{2}) - f(\boldsymbol{x}^{*})\right],   \label{equ:conv_15}  
\end{align}
where $\boldsymbol{x}^{*}$ represents the optimal solution. 
If $T$ is chosen to be  
\begin{align}
T > \frac{ 2Ld \left[ \alpha^q \sigma^q+\xi  (1-\beta)^q \right]^2 \xi^2}{ \epsilon^2 \left(  \alpha^q\sigma^q   \beta (1-\beta) \theta_1 \theta_2  \right)^2}  \left[ f(\boldsymbol{x}_{2}) - f(\boldsymbol{x}^{*})\right], \nonumber
\end{align}
the RHS of (\ref{equ:conv_15}) is upper bounded by $ \epsilon^2$, which violates the assumption of $\{\|\boldsymbol{g}_t\|>\epsilon |t\geq 1  \} $. The proof is complete.

\ifCLASSOPTIONcaptionsoff
  \newpage
\fi

\bibliographystyle{IEEEtran}

\end{document}